\documentclass[letterpaper, 10 pt, conference]{ieeeconf}  % Comment this line out if you need a4paper

\IEEEoverridecommandlockouts                              % This command is only needed if
\usepackage{multicol}
\usepackage{graphicx}
\usepackage{dblfloatfix}
\usepackage[ruled]{algorithm2e}

\SetCommentSty{mycommfont}
\SetFuncSty{myprocname}

\usepackage{amsmath}
\usepackage{bm}
\usepackage{amssymb}
\usepackage{upgreek}
\usepackage{multirow}
\usepackage{amsfonts}
\usepackage{here}
\usepackage[normalem]{ulem}
\usepackage{booktabs}
\usepackage{color}
\usepackage{subcaption}
\usepackage{url}
\usepackage[font=small, skip=5pt]{caption}
\usepackage{enumerate}
\usepackage{listings}
\usepackage{setspace}
\usepackage[bookmarks=true]{hyperref}
\usepackage{xcolor} % Hyperlink color
\hypersetup{
    colorlinks,
    linkcolor={red!50!black},
    citecolor={blue!80!black},
    urlcolor={blue!80!black}
}

\usepackage{cleveref}

\makeatletter
\let\NAT@parse\undefined
\makeatother
\usepackage[numbers]{natbib}

\DeclareMathOperator*{\argmax}{argmax}

\newcommand{\E}{\mathbb{E}}

\graphicspath{{figs/}}

\newtheorem{lemma}{Lemma}

\newcommand{\Sspace}{\mathcal{S}}
\newcommand{\Aspace}{\mathcal{A}}
\newcommand{\Ospace}{\mathcal{O}}
\newcommand{\MOVE}{\textsc{Move}}
\newcommand{\MOVEOP}{\textsc{MoveOp}}
\newcommand{\LOOK}{\textsc{Look}}
\newcommand{\DETECT}{\textsc{Find}}

\newcommand{\Val}{\textsc{Val}}
\newcommand{\Norm}{\textsc{Norm}}
\newcommand{\Children}{\textsc{Ch}}
\newcommand{\abst}[1]{\hat{#1}}

\newcommand{\rmax}{R_{\text{max}}}
\newcommand{\rmin}{R_{\text{min}}}

\definecolor{purple}{rgb}{0.858, 0.08, 0.85}
\definecolor{darkgreen}{rgb}{0.08, 0.55, 0.08}
\definecolor{blue1}{rgb}{0.2, 0.2, 0.6}
\definecolor{green1}{rgb}{0.2, 0.6, 0.2}
\definecolor{green2}{rgb}{0.1, 0.4, 0.1}

\begin{document}

% paper title
\title{\LARGE \bf
Multi-Resolution POMDP Planning for Multi-Object Search in 3D}

\author{
  Kaiyu Zheng$^{\dagger}$, Yoonchang Sung$^{*}$, George Konidaris$^{\dagger}$, Stefanie Tellex$^{\dagger}$
        \thanks{$^{\dagger}$Brown University, Providence, RI. $^{*}$MIT CSAIL, Cambridge, MA.}
        \thanks{Email: \{kzheng10, gdk, stefie10\}@cs.brown.edu, yooncs8@csail.mit.edu}
    }

%%%%%%%%%%%%%%%%%%%%%%%%%%%%% title and abstract %%%%%%%%%%%%%%%%%%%%%%%%%%%%%%%%%%%%%%
\maketitle

\begin{abstract}
  Robots operating in households must find objects on shelves, under tables, and in cupboards. In such environments, it is crucial to search efficiently at 3D scale while coping with limited field of view and the complexity of searching for multiple objects. Principled approaches to object search frequently use Partially Observable Markov Decision Process (POMDP) as the underlying framework for computing search strategies, but constrain the search space in 2D. In this paper, we present a POMDP formulation for multi-object search in a 3D region with a frustum-shaped field-of-view. To efficiently solve this POMDP, we propose a multi-resolution planning algorithm based on online Monte-Carlo tree search. In this approach, we design a novel octree-based belief representation to capture uncertainty of the target objects at different resolution levels, then derive abstract POMDPs at lower resolutions with dramatically smaller state and observation spaces. Evaluation in a simulated 3D domain shows that our approach finds objects more efficiently and successfully compared to a set of baselines without resolution hierarchy in larger instances under the same computational requirement. We demonstrate our approach on a mobile robot to find objects placed at different heights in two 10m$^2\times$2m regions by moving its base and actuating its torso.
\end{abstract}

%%%%%%%%%%%%%%%%%%%%%%%%%%%%%%%%%%%%%% main body %%%%%%%%%%%%%%%%%%%%%%%%%%%%%%%%%%%%%%

%%%%%%%%%%%%%%%%%
%%% Main Body %%%
%%%%%%%%%%%%%%%%%

\section{Introduction}
\label{sec:introduction}
% Importance of the problem such as homes
Robots operating in human spaces must find objects such as glasses, books, or cleaning supplies that could be on the floor, shelves, or tables. This search space is naturally 3D. When multiple objects must be searched for, such as a cup and a mobile phone, an intuitive strategy is to first hypothesize likely search regions for each target object based on semantic knowledge or past experience~\cite{kollar2009utilizing,aydemir2013avo}, then search carefully within those regions. Since the latter directly determines the success of the search, it is essential for the robot to produce an efficient search policy within a designated search region under limited field of view (FOV), where target objects could be partially or completely occluded. In this work, we consider the problem setting where a robot must search for multiple objects in a search region by actively moving its camera, with as few steps as possible (Figure~\ref{fig:mos3d}).

Searching for objects in a large search region requires acting over long horizons under various sources of uncertainty in a partially observable environment. For this reason, previous works have used Partially Observable Markov Decision Process (POMDP) as a principled decision-theoretic framework for object search \cite{xiao_icra_2019,atanasov2014nonmyopic,danielczuk2019mechanical}.  However, to ensure the POMDP is manageable to solve, previous works reduce the search space or robot mobility to 2D \cite{aydemir2013avo,wandzel2019multi,li2016act}, although objects exist in rich 3D environments. The key challenges lie in the intractability of maintaining exact belief due to large state space \cite{silver2010monte}, and the high branching factor for planning due to large observation space \cite{sunberg2018pomcpow,garg2019despot}.

\begin{figure}[t]
\centering
\includegraphics[width=\linewidth]{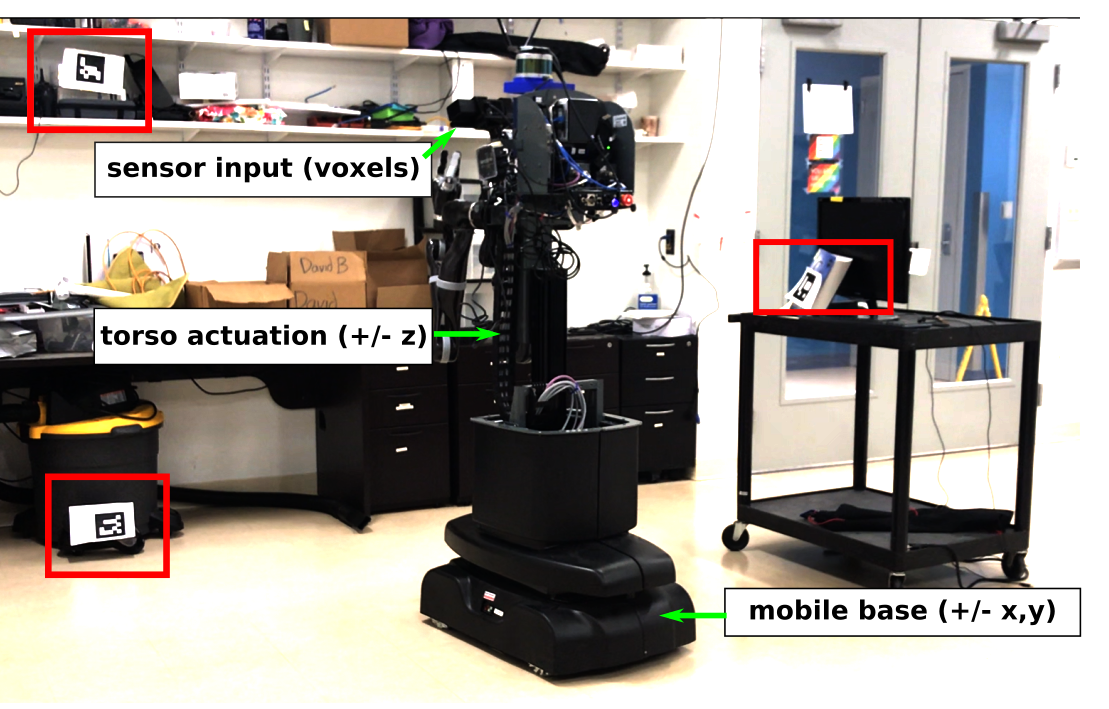}
\caption{An example of the 3D-MOS problem where a torso-actuated mobile robot is
  tasked to search for three objects placed at different heights in a lab
  environment. The objects are represented by paper AR tags marked by red
  boxes. Note that the robot must actively move itself due to limited field of
  view, and the objects can be occluded by the attached obstacles if viewed from
  the side.}
\label{fig:mos3d}
\end{figure}

In this paper, we introduce 3D Multi-Object Search (3D-MOS), a general POMDP formulation for the multi-object search task with 3D state and action spaces, and a realistic observation space in the form of labeled voxels within the viewing frustum from a mounted camera. Following the Object-Oriented POMDP (OO-POMDP) framework proposed by \citet{wandzel2019multi}, the state, observation spaces are factored by independent objects, allowing the belief space to scale linearly instead of exponentially in the number of objects. We address the challenges of computational complexity in solving 3D-MOS by developing several techniques that converge to an online multi-resolution planning algorithm. First, we propose a per-voxel observation model which drastically reduces the size of the observation space necessary for planning. Next, we present a novel octree-based belief representation that captures beliefs at different resolutions and allows efficient and exact belief updates. Then, we exploit the octree structure and derive abstractions of the ground problem at different resolution levels leveraging abstraction theory for MDPs \cite{li2006towards,bai2016markovian}.  Finally, a Monte-Carlo Tree Search (MCTS) based online planning algorithm, called Partially-Observable Upper Confidence bounds for Trees (POUCT) \cite{silver2010monte}, is employed to solve these abstract instances in parallel, and the action with highest value in its MCTS tree is selected for execution.

We evaluate the proposed approach in a simulated, discretized 3D domain where a robot with a 6 degrees-of-freedom camera searches for objects of different shapes and sizes randomly generated and placed in a grid environment. The results show that, as the problem scales, our approach outperforms exhaustive search as well as POMDP baselines without resolution hierarchy under the same computational requirement. We also show that our method is more robust to sensor uncertainty against the POMDP baselines.  Finally, we demonstrate our approach on a torso-actuated mobile robot in a lab environment (Figure~\ref{fig:seq}). The robot finds 3 out of 6 objects placed at different heights in two 10m$^2\times$2m regions in around 15 minutes.

%\footnote{Code: \url{https://github.com/zkytony/mos3d}}

%%%%% MAKE A POINT %%%%%%
% Our approach focuses on searching within a region and can
% be combined with other approaches that impose prior knowledge

\section{Background}
\label{sec:related_work}

POMDPs compactly represent the robot's uncertainty in target locations and its
own sensor \cite{kaelbling1998planning}, and OO-POMDPs factor the domain in terms of objects, which fits the
object search problem naturally \cite{wandzel2019multi}. Below, we first provide a brief overview of POMDPs and OO-POMDPs. Then, we discuss related work in object search.

\subsection{POMDPs and OO-POMDPs}
\label{sec:pomdp}

% A POMDP models a sequential decision making problem where the environment state is
% not fully observable.
A POMDP models a sequential decision making problem where the environment state
is not fully observable by the agent. It is formally defined as a tuple
$\langle\Sspace,\Aspace,\Ospace,T,O,R,\gamma\rangle$, where
$\Sspace,\Aspace,\Ospace$ denote the state, action and observation spaces, and the functions $T(s,a,s')=\Pr(s'|s,a)$, $O(s',a,o)=\Pr(o|s',a)$, and
$R(s,a)\in\mathbb{R}$ denote the transition, observation, and reward models. The
agent takes an action $a\in\mathcal{A}$ that causes the environment state to transition from
$s\in\mathcal{S}$ to $s'\in\mathcal{S}$. The environment in turn returns the agent an observation $o\in\mathcal{O}$ and
reward $r\in\mathbb{R}$. A \emph{history} $h_t=(ao)_{1:t-1}$ captures all past actions and
observations. The agent maintains a distribution over states given current
history $b_t(s)=\Pr(s|h_t)$.
The agent updates its belief after taking an action and receiving an
observation by $b_{t+1}(s')=\eta\Pr(o|s',a)\sum_{s}\Pr(s'|s,a)b_t(s)$
where $\eta={\sum_s\sum_{s'}\Pr(o|s',a)\Pr(s'|s,a)b_t(s)}$
is the normalizing constant.
The task of the agent is to find a policy $\pi(b_t)\in\Aspace$ which maximizes the
expectation of future discounted rewards ${V^{\pi}}(b_t)=\E\left[\sum_{k=0}^{\infty}\gamma^{k}R(s_{t+k},\pi(b_{t+k}))\ |\ b_t\right]$
with a discount factor
$\gamma$.

An Object-Oriented POMDP (OO-POMDP) \cite{wandzel2019multi} (generalization of OO-MDP
  \cite{diuk2008object}) is a POMDP that considers the state
  and observation spaces to be factored by a set of $n$ objects where each belongs to a class with a set of attributes. A simplifying assumption is made for the 2D MOS domain that objects are independent so that the belief space scales linearly rather than exponentially in the number of objects: $b_t(s)=\prod_ib_t^i(s_i)$. We make this assumption for the same computational reason.

Offline POMDP solvers are often too slow to be practical for large domains
\cite{ross2008online}. State-of-the-art online POMDP solvers leverage sparse
belief sampling and MCTS to scale up to domains with large state spaces and to address the curse of history~\cite{silver2010monte,somani2013despot,sunberg2018pomcpow}.  POMCP
\cite{silver2010monte} is one such algorithm which combines particle belief
representation with Partially Observable UCT (POUCT), which extends the UCT
algorithm \cite{kocsis2006bandit} to POMDPs and is proved to be asymptotically optimal \cite{silver2010monte}. We build upon POUCT due to its optimality and simplicity of implementation.

\subsection{Related Work}
Previous work primarily address the computational complexity of object search by
hypothesizing likely regions based on object co-occurrence \cite{kollar2009utilizing,wixson1994using}, semantic knowledge \cite{aydemir2013avo} or language \cite{wandzel2019multi}, reducing the state space from 3D to 2D \cite{wandzel2019multi,wang2018efficient,sarmiento2003efficient,nie2016searching}, or constrain the sensor to be stationary \cite{danielczuk2019mechanical,dogar2014object}. Our work focuses on multi-object search within a 3D region where the robot actively moves the mounted camera, for example, through pan or tilt, or by moving itself.

Several works explicitly reason over the arrangement of occluded objects based on given geometry models of clutter \cite{xiao_icra_2019,nie2016searching,wong2013manipulation}. Our approach considers occlusion as part of the observation that contains no information about target locations and we do not require geometry models.

Many works formulate object search as a POMDP. Notably, \citet{aydemir2013avo} first infer a room to search in then perform search by calculating candidate viewpoints in a 2D plane.
\citet{li2016act} plan sensor movements online, yet assume objects are placed at the same surface level in a container with partial occlusion. \citet{xiao_icra_2019} address object fetching on a cluttered tabletop where the robot's FOV fully covers the scene, and that occluding obstacles are removed permanently during search. \citet{wandzel2019multi} formulates the multi-object search (MOS) task on a 2D map using the proposed Object-Oriented POMDP (OO-POMDP). We extend that work to 3D and tackle additional challenges by proposing a new observation model and belief representation, and a multi-resolution planning algorithm. In addition, our POMDP formulation allows fully occluded objects and can be in principle applied on different robots such as mobile robots or drones.

%%%%%%%%%%%%%%%%%%%%%%%%%%
%%% Technical Approach %%%
%%%%%%%%%%%%%%%%%%%%%%%%%%
\section{Multi-Object Search in 3D}
\label{sec:mos3d}

The robot is tasked to search for $n$ static target objects (e.g. cup and book) of known type but unknown location in a search space that also contains static non-target obstacles. We assume the robot has access to detectors for the objects that it is searching for. The search region is a 3D grid map environment denoted by $G$. Let $g\in G\subseteq{R}^3$ be a 3D grid cell in the environment. We use $G^l$ to denote a grid at \emph{resolution level} $l\in\mathbb{N}$, and $g^l\in G^l$ to denote a grid cell at this level. When $l$ is omitted, it is assumed that $g$ is at the ground resolution level. We introduce the 3D-MOS domain as an OO-POMDP as follows:

\textbf{State space $\Sspace$.} An environment state $s=\{s_1,\cdots,s_n,s_r\}$ is factored in an object-oriented way, where $s_r\in\mathcal{S}_r$ is the state of the robot, and $s_i\in\mathcal{S}_i$ is the state of target object $i$. A robot state is defined as $s_r=(p,\mathcal{F})\in\mathcal{S}_r$ where $p$ is the 6D camera pose and $\mathcal{F}$ is the set of found objects. The robot state is assumed to be observable to the robot. In this work, we consider the object state to be specified by one attribute,  the 3D object pose at its center of mass, corresponding to a cell in grid $G$. We denote a state $s_i^l\in\mathcal{S}_i^l$ to be an object state at resolution level $l$, where $\mathcal{S}_i^l=G^l$.

\begin{figure}[t]
\centering
\includegraphics[width=0.75\linewidth]{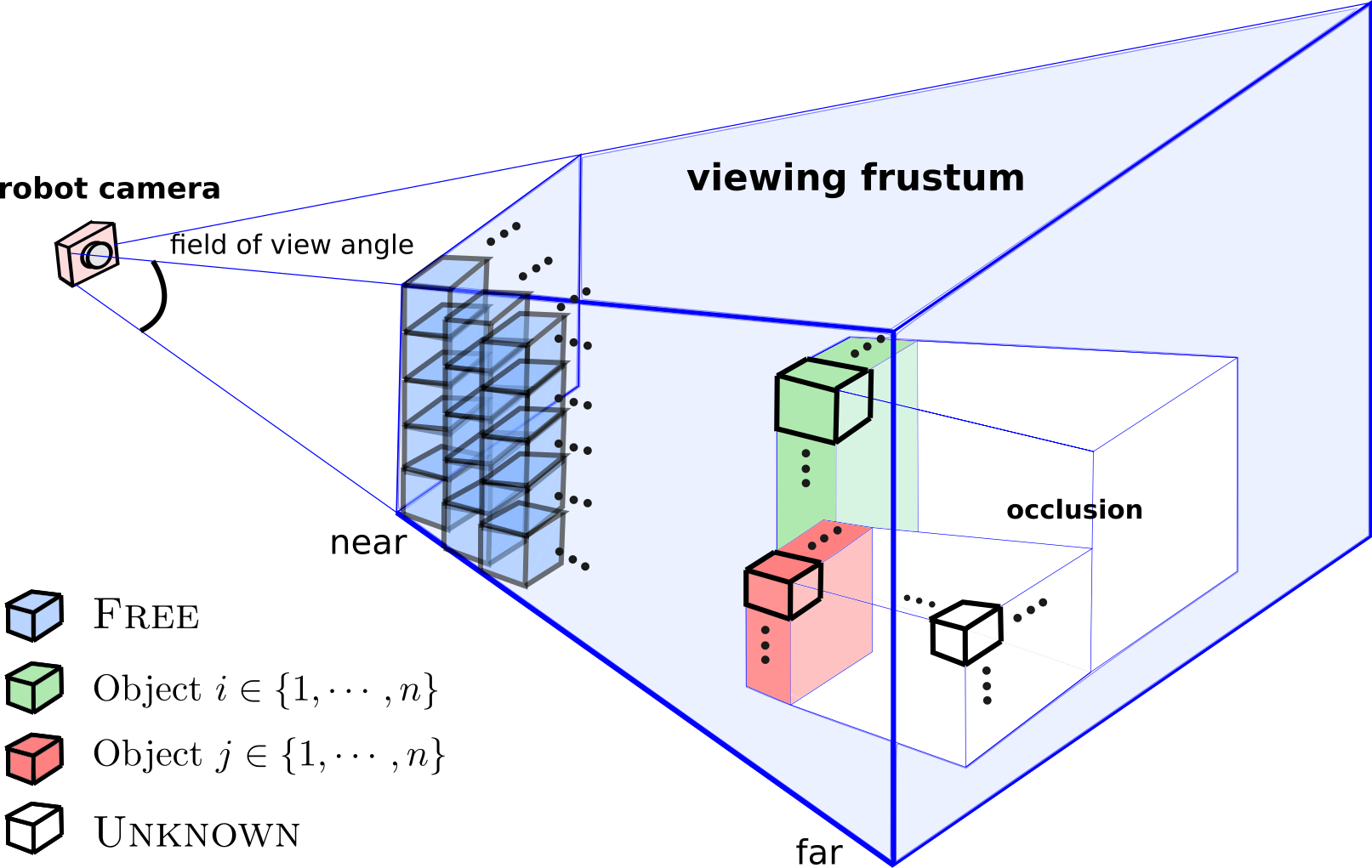}
\caption{Illustration of the viewing frustum and volumetric observation. The
  viewing frustum $V$ consists of $|V|$ voxels, where each $v\in V$ can be
  labeled as $i\in\{1,\cdots,n\}$, \textsc{Free} or \textsc{Unknown}.}
\label{fig:fov}
\vspace{-1em}
\end{figure}

% Observation
\textbf{Observation space $\Ospace$.}
% The robot receives an observation through
% a projected frustum FOV $V$, which consists of $|V|$ voxels.
The robot receives an observation through a viewing frustum projected from a mounted camera. The viewing frustum forms the FOV of the robot, denoted by $V$, which consists of $|V|$ voxels. Note that the resolution of a voxel should be no lower than that of a 3D grid cell $g$. We assume both resolutions to be the same in this paper for notational convenience, hence $V\subseteq G$, but in general a voxel with higher
resolution can be easily mapped to a corresponding grid cell.

For each voxel $v\in V$, a \emph{detection function} $d(v)$ labels the voxel to be either an object $i\in\{1,\cdots,n\}$, \textsc{Free}, or \textsc{Unknown} (Figure~\ref{fig:fov}). \textsc{Free} denotes that the voxel is a free space or an obstacle. We include the label $\textsc{Unknown}$ to take into account occlusion incurred by target objects or static obstacles. In this case, the corresponding voxel in $V$ does not give any information about the environment. An observation $o=\{(v,d(v)) | v\in V\}$ is defined as a set of voxel-label tuples. This can be thought of as the result of voxelization and object segmentation given a raw point cloud.

% where $V$ is the viewing frustum that represents the FOV , and $d(v)$ is a \emph{detection function} that labels the voxel $v\in V$ to be either an object $i\in\{1,\cdots,n\}$, \textsc{Free} or \textsc{Unknown}.
% \footnote{This can be thought
% of as the result of voxelization and object segmentation given the raw point
% cloud.}.

% We let $V\subseteq G$ for
% notational convenience, but in general a voxel with higher
% resolution can be easily mapped to a corresponding grid cell.

We can factor $o$ by objects in the following way. First, given the robot state $s_r$ at which $o$ is received, the voxels in $V$ have known locations. Under this condition, $V$ can be reduced to exclude voxels labeled \textsc{Unknown} while still maintaining the same information. Then, $V$ can be decomposed by objects into $V_1,\cdots,V_n$ where for any $v\in V_i$, $d(v)\in\{i,\textsc{Free}\}$ which retain the same information as $V$ for a given robot state.\footnote{The FOV $V$ is fixed for a given camera pose in the robot state, therefore excluding \textsc{Unknown} voxels does not lose information.} Hence, the observation $o=\bigcup_{i=1}^n o_i$ where $o_i=\{(v,d(v))|v\in V_i\}$.

%This is illustrated in Fig.

% Actions
\textbf{Action space $\Aspace$.}
Searching for objects generally requires three basic
capabilities: \emph{moving}, \emph{looking}, and \emph{declaring} an object to be found at some location.
% The
%   robot receives an observation through looking, and signals a commitment
%   to its belief of the object location through declaring.
Formally, the action space consists of these three types of primitive actions: $\MOVE(s_r,g)$ action moves the robot from pose in $s_r$ to destination $g\in G$ stochastically. $\LOOK(\theta)$ changes the
camera pose to look in the direction specified by $\theta\in\mathbb{R}^3$, and
projects a viewing frustum $V$. $\DETECT(i,g)$ declares object $i$ to be found
at location $g$. The implementation of these actions may vary depending on the
type of search space or robot. Note that this formulation allows macro actions, such as ``look after move'' to be composed for planning.
%%%%%%%%% Do we need the following?
% \todo{needed?}
% This assumption is often used in prior work \cite{wandzel2019multi, xiao_icra_2019,
%   bandyopadhyay2013intention} which separates the localization problem from the
%   already challenging POMDP planning problem.
%All transitions are assumed to be deterministic.

\textbf{Transition function $T$.} Target objects and obstacles are static objects, thus $\Pr(s_i'|s,a)=\bm{1}(s_i' = s_i)$. For the robot, the actions \textsc{Move}($s_r$, $g$) and $\LOOK(\theta)$ change the camera location and direction to $g$ and $\theta$ following a domain-specific stochastic dynamics function. The action $\DETECT(i,g)$ adds $i$ to the set of \emph{found objects} in the robot state only if $g$ is within the FOV determined by $s_r$.
% In our evaluation, as in previous work \cite{}, we consider a domain with deterministic robot state transition to separate the localization problem from the object search problem.

\textbf{Reward function $R$.} The correctness of declarations can only be determined by, for example, a human who has knowledge about the target object or additional interactions with the object; therefore, we consider declarations to be expensive. The robot receives $\rmax\gg 0$ if an object is correctly identified by a $\DETECT$
action, otherwise the $\DETECT$ action incurs a $\rmin\ll 0$ penalty. $\MOVE$ and
$\LOOK$ receive a negative step cost $R_{\text{step}}<0$ dependent on the robot state and the action itself. This is a sparse reward function.

\subsection{Observation Model}
\label{sec:observation_model}
We have previously described how a volumetric observation $o$ can be factored by objects into $o_1,\cdots,o_n$. Here, we describe a method to model $\Pr(o_i|s',a)$, the probabilistic distribution over an observation $o_i$ for object $i$.

Modeling a distribution over a 3D volume is a challenging problem \cite{park2019deepsdf}. To develop an efficient model, we make the simplifying assumption that object $i$ is contained within a single voxel located at the grid cell $g=s_i'$. We address the case of searching for objects of unknown sizes with our planning algorithm (Section~\ref{sec:planning_algorithm}) that plans at multiple resolutions in parallel.

Under this assumption, $d(v)=\textsc{Free}$ deterministically for $v\neq s_i'$, and the uncertainty of $o_i$ is reduced to the uncertainty of $d(s_i')$. As a result, $\Pr(o_i|s',a)$ can be simplified to $\Pr(d(s_i)|s',a)$. When $s_i'\not\in V_i$, either $d(s_i')=\textsc{Unknown}$ (occlusion) or $s_i'\not\in V$ (not in FOV). In this case, there is no information regarding the value of $d(s_i')$ in the observation $o_i$, therefore $\Pr(d(s_i')|s',a)$ is a uniform distribution. When $s_i'\in V_i$, that is, the non-occluded region within the FOV covers $s_i'$, the case of $d(s_i')=i$ indicates correct detection while $d(s_i')=\textsc{Free}$ indicates sensing error. We let $\Pr(d(s_i')=i|s',a)=\alpha$ and $\Pr(d(s_i')=\textsc{Free}|s',a)=\beta$. It should be noted that $\alpha$ and $\beta$ do not necessarily sum to one because the belief update equation does not require the observation model to be normalized, as explained in Section~\ref{sec:pomdp}. Thus, hyperparameters $\alpha$ and $\beta$ independently control the reliability of the observation model.

\section{Octree Belief Representation}
\label{sec:octree_belief}

Particle belief representation \cite{silver2010monte,somani2013despot} suffers
from particle depletion under large observation spaces.  Moreover, if the resolution
of $G$ is dense, it may be possible that most of 3D grid cells do not contribute
to the behavior of the robot.

We represent a belief state $b_t^i(s_i)$ for object $i$ as an octree, referred to as an \emph{octree belief}. It can be constructed incrementally as observations are received and it tracks the belief of object state at different resolution levels. Furthermore, it allows efficient belief sampling and belief update using the per-voxel observation model (Sec.~\ref{sec:observation_model}).

An octree belief consists of an octree and a normalizer. An octree is a tree where every node has $8$ children. In our context, a node represents a grid cell $g^l\in G^l$, where $l$ is the resolution level, such that $g^l$ covers a cubic volume of $(2^l)^3$ ground-level grid cells; the ground resolution level is given by $l=0$. The 8 children of the node equally subdivide the volume at $g^l$ into smaller volumes at resolution level $l-1$ (Figure~\ref{fig:octree}). Each node stores a value $\Val_t^i(g^l)\in\mathbb{R}$, which represents the unnormalized belief that $s_i^l=g^l$, that is, object $i$ is located at grid cell $g^l$.  We denote the set of nodes at resolution level $k<l$ that reside in a subtree rooted at $g^l$ by $\Children^{k}(g^l)$. By definition, $b_t^i(g^l)=\Pr(g^l|h_t)=\sum_{c\in\Children^{k}(g^l)}\Pr(c|h_t)$. Thus, with a normalizer $\Norm_t=\sum_{g\in G}\Val_t^i(g)$, we can rewrite the normalized belief as:
\begin{align}
  \label{eq:octbdef}
b_t^i(g^l)=\frac{\Val_t^i(g^l)}{\Norm_t}=\sum_{c\in\Children^{k}(g^l)}\left(\frac{\Val_t^i(c)}{\Norm_t}\right),
\end{align}
which means $\Val_t^i(g^l)=\sum_{c\in\Children^{k}(g^l)}\Val_t^i(c)$. In words, the value stored in a node is the sum of values stored in its children. The normalizer equals to the sum of values stored in the nodes at the ground resolution level.

The octree does not need to be constructed fully in order to query the
probability at any grid cell. This can be achieved by setting a default value
$\Val_0^i(g)=1$ for all ground grid cells $g\in G$ not yet present in
the octree. Then, any node corresponding to $g^l$ has a default value of
$\Val_0^i(g^l)=\sum_{c\in\Children^1(g^l)}\Val_0^i(c)=|\Children^1(g^l)|$.

% Next, we describe the belief update and sampling algorithm for the
% octree belief representation.

\begin{figure}[t]
\centering
\includegraphics[width=0.7\linewidth]{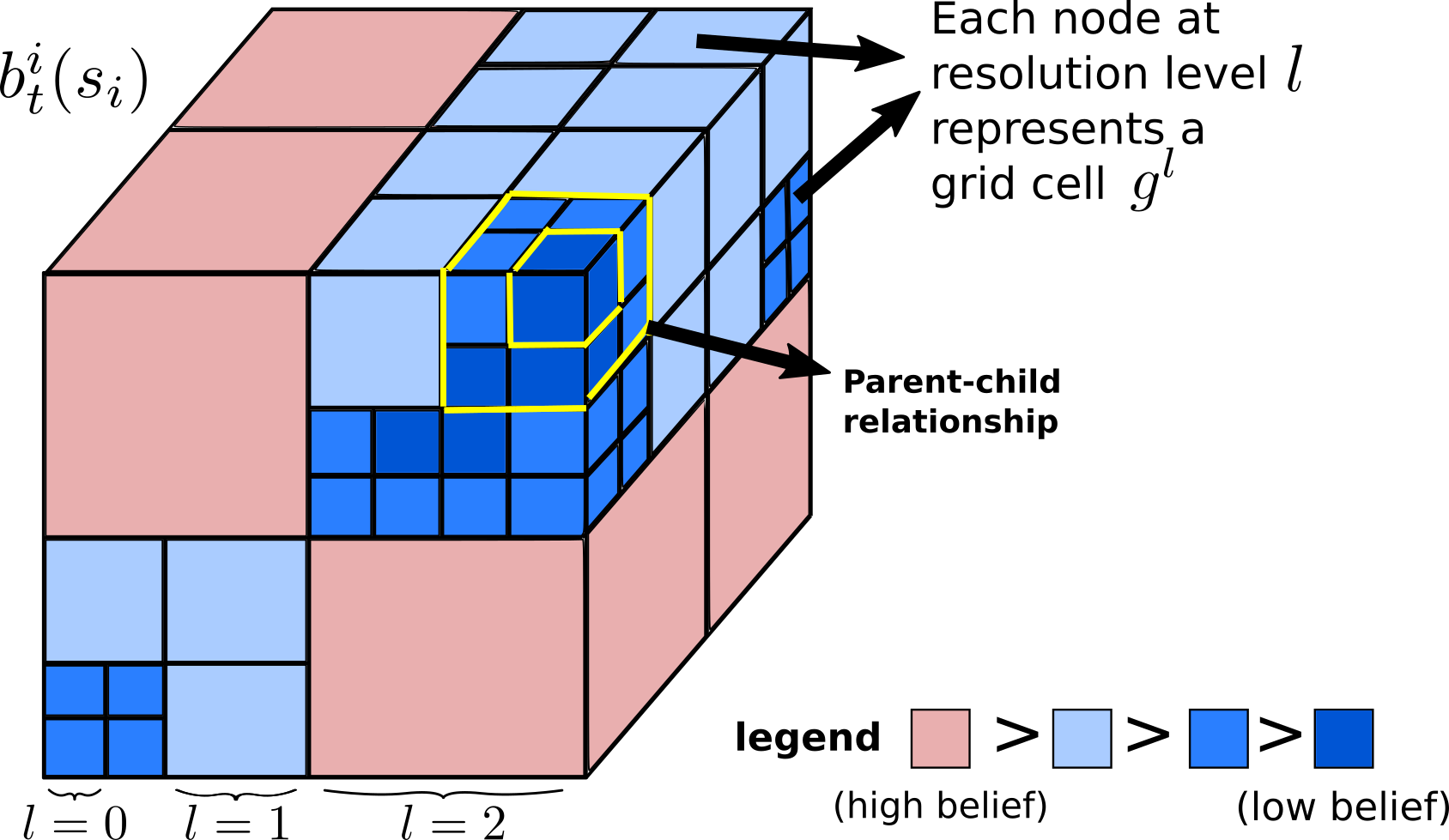}
\caption{Illustration the octree belief representation $b_t^i(s_i)$. The color
  on a node $g^l$ indicates the belief $\Val_t^i(g^l)$ that the object is
  located within $g^l$. The highlighted grid cells indicate parent-child
  relationship between a grid cell at resolution level $l=1$ (parent) and one at
  level $l=0$.}
\label{fig:octree}
\vspace{-1em}
\end{figure}

\subsection{Belief Update}
\label{sec:octree:belief_update}
We have defined a per-voxel observation model for $\Pr(o_i|s',a)$, which is reduced to $\Pr(d(s_i')|s',a)$ if $s_i'\in V_i$, or a uniform distribution if $s_i'\not\in V_i$. This suggests that the belief update need only
happen for voxels that are inside the FOV to reflect the information in the
observation.
% To this end, we propose the following belief update algorithm for $\Val_t^i(s_i)$ if $s_i\in V_i$, and the normalizer $\Norm_t$:

Upon receiving observation $o_i$ within the FOV
$V_i$, belief is updated according to Algorithm \ref{alg:belief_update}.
This algorithm updates the value of the ground-level node $g$
corresponding to each voxel $v\in V_i$ as $\Val_{t+1}^i(g)=\Pr(d(v)|s',a)\Val_{t}^i(g)$. The normalizer is updated to make sure $b_{t+1}^i$ is normalized
\begin{lemma}
  The normalizer $\Norm_t$ at time $t$ can be correctly updated by adding the
  incremental update of values as in Algorithm \ref{alg:belief_update}.
  % In other words, the following holds:;
  % \begin{align}
  %   \label{eq:norm_update}
    % $$\Norm_{t+1} = \Norm_{t} + \sum_{s_i\in V_i}\left(\Val_{t+1}^i(s_i) - \Val_{t}^i(s_i)\right)$$
%   \end{align}
\end{lemma}
\begin{proof}
  \label{pf:norm_proof}
The normalizer must be equal to the sum of node values at the ground level for the next belief $b_{t+1}^i$ to be valid (Equation~\ref{eq:octbdef}). That is, $\Norm_{t+1}=\sum_{s_i\in G}\Val^i_{t+1}(s_i)$. This sum can be decomposed into two cases where the object $i$ is inside of $V_i$ and outside of $V_i$; For object locations $s_i\not\in V_i$, the \emph{unnormalized} observation model is uniform, thus $\Val^i_{t+1}(s_i)=\Pr(d(s_i)|s',a)\Val^i_{t}(s_i)=\Val^i_{t}(s_i)$. Therefore, $\Norm_{t+1}=\sum_{s_i\in V_i}\Val^i_{t+1}(s_i) + \sum_{s_i\not\in V_i}\Val^i_{t}(s_i)$. Note the set $\{s_i | s_i\not\in V_i\}$ is equivalent as $\{s_i | s_i \in G\setminus V_i\}$. Using this fact and the definition of $\Norm_{t}$, we obtain
$\Norm_{t+1}=\Norm_t+\sum_{s_i\in V_i}\left(\Val^i_{t+1}(s_i) - \Val^i_{t}(s_i)\right)$
which proves the lemma.
\end{proof}

This belief update is therefore exact since the objects are static.
The complexity of this algorithm is $O(|V|\log(|G|)$;
% where $|V|$ is the size of the viewing frustum and $|G|$ is the size of the search space;
 Inserting nodes and updating values of nodes can be done by traversing the tree depth-wise.

\begin{algorithm}[t]
\caption{OctreeBeliefUpdate $(b_t^i,a,o_i)\rightarrow b_{t+1}^i$}
\label{alg:belief_update}
\SetKwInOut{Input}{input}
\SetKwInOut{Output}{output}
\Input{$b_t^i$: octree belief for object $i$; $a$: action taken by robot; $o_i=\{(v,d(v)|v\in V_i\}$: factored observation for object $i$}
\Output{$b_{t+1}^i$: updated octree belief}
\tcp{Let $\Psi(b_i^t)$ denote the octree underlying $b_t^i$. }
\For{$v\in V_i$}{
  $s_i\gets v$\tcp*{State at grid cell corresponding to voxel $v$}
  \If{$s_i\not\in\Psi(b_i^t)$}{
    Insert node at $s_i$ to $\Psi(b_i^t)$\;
  }
  % \tcp{Now computing the belief update $b_{t+1}^i(s_i)$}
  $\Val_{t+1}^i(s_i)\gets \Pr(d(v)|s',a)\Val_{t}^i(s_i)$\;
  $\Norm_{t+1}\gets\Norm_{t} + \Val_{t+1}^i(s_i) - \Val_{t}^i(s_i)$\;
}
\end{algorithm}

\subsection{Sampling}
Octree belief affords exact belief sampling at any resolution level in
logarithmic time complexity with respect to the size of the search space
$|G|$, despite not being completely built.
Given resolution level $l$, we sample from $\mathcal{S}_i^l$ by traversing the
octree in a depth-first manner.  Let $l_{max}$ denote the maximum resolution
level for the search space. Let $l_{des}$ be the \emph{desired} resolution level
at which a state is sampled. If $s_i^{l_{des}}$ is sampled, then all nodes in
the octree that cover $s_i^{l_{des}}$ , i.e, $s_i^{l_{max}},\cdots,s_i^{l_{des}+2},s_i^{l_{des}+1}$, must also be implicitly sampled. Also, the event that $s_i^{l+k}$ is sampled is independent from other samples given that $s_i^{l+k+1}$ is sampled. Hence,
the task of sampling $s^{l_{des}}$ is translated into sampling a
sequence of samples
$s_i^{l_{max}},\cdots,s_i^{l_{des}+2},s_i^{l_{des}+1},s_i^{l_{des}}$, each
according to the distribution
$\Pr(s_i^{l}|s_i^{l+1},h_t)=\frac{\Val_t^i(s_i^{l})}{\Val_t^i(s_i^{l+1})}$. Sampling
from this probability distribution is efficient, as the sample space, i.e. the
children of node $s_i^{l+1}$ is only of size 8. Therefore, this sampling scheme
yields a sample $s^{l_{des}}$ exactly according to $b_t^{i}(s^{l_{des}})$ with
time complexity $O(\log(|G|))$.

\section{Multi-Resolution Planning via Abstractions}
  \label{sec:planning_algorithm}

% \stnote{I would move this to where you actually talk about the solver rather than here in % preliminaries.}

POUCT expands an MCTS tree using a \emph{generative function} $(s',o,r)\sim\mathcal{G}(s,a)$,
which is straightforward to acquire since we explicitly define the 3D-MOS models. However,
directly applying POUCT is subject to high branching factor due to the large observation space in our domain.

Our intuition is that octree belief imposes a spatial state abstraction,
which can be used to derive an abstraction over observations, reducing the branching factor for planning.
Below, we formulate an
\emph{abstract 3D-MOS} with smaller spaces, and propose our multi-resolution planning algorithm.

\subsection{Abstract 3D-MOS}
\label{sec:abstract_3dmos}
We adopt the abstraction scheme in \citet{li2006towards} where in general, the abstract transition and reward functions are weighted sums of the original problem's transition and reward functions, respectively with weights that sum up to~1. We define an abstract 3D-MOS
$\langle \hat{\Sspace}, \hat{\Aspace}, \hat{\Ospace},
\hat{T},\hat{O},R,\gamma,l\rangle$ at resolution level $l$ as follows.

\textbf{State space $\hat{\mathcal{S}}$.} For each object $i$, an abstraction function $\phi_i:\Sspace_i\rightarrow\Sspace_i^l$ transforms the ground-level object state $s_i$ to an abstract object state $s_i^l$ at resolution level $l$. The abstraction of the full state is $\hat{s}=\phi(s)=\{s_r\}\cup\bigcup_i \phi_i(s_i)$ where the robot state $s_r$ is kept as is. The \emph{inverse image} $\phi^{-1}_i(s_i^l)$ is the set of ground states that correspond to $s_i^l$ under $\phi_i$ \cite{li2006towards}.

\textbf{Action space $\hat{\Aspace}$.} Since state abstraction lowers the resolution of the search space, we consider macro move actions that move the robot over longer distance at each planning step. Each macro move action \textsc{MoveOp}$(s_r,g)$ is an \emph{option} \cite{sutton1999between} that moves $s_r$ to goal location $g$ using multiple \textsc{Move} actions. The primitive \textsc{Look} and \textsc{Find} actions are kept.

\textbf{Transition function $\hat{T}.$} Targets and obstacles are still static, and the robot state still transitions according to the ground-level transition function. However, the transition of the found set from $\mathcal{F}$ to $\mathcal{F}'$ is special since the action $\textsc{Find}(i,g)$ operates at the ground level while $s_i^l$ has a lower resolution ($l>0$).
% \footnote{$\textsc{Find}(i,g^l)$ is invalid because the original 3D-MOS task has no reward defined for such an action and it would lead to incorrect search.}.
Let $f_i$ be the binary state variable that is true if and only if object $i\in\mathcal{F}$. Because the action \textsc{Find}$(i,g)$ affects $f_i$ based only on whether object $i$ is located at $g$, and that the problem is no longer Markovian due to state abstraction \cite{bai2016markovian}, $f_i$ transitions to $f_i'$ following
\begin{align}
\begin{split}
&\Pr(f_i'|f_i,s_i^l,h_t,\textsc{Find}(i,g))
% &\ \ =\sum_{s_i\in\phi^{-1}_i(s_i^l)}\Pr(s_i,f'_i | f_i,s_i^l,h_t,\textsc{Find}(i,g))
\end{split}\\
&\ \ =\sum_{s_i\in\phi^{-1}_i(s_i^l)}\Pr(f'_i | s_i, f_i,\textsc{Find}(i,g))\Pr(s_i|s_i^l,h_t).
\end{align}
The above is consistent with the abstract transition function in the works \cite{li2006towards,bai2016markovian} where the first term corresponds to the ground-level deterministic transition function and the second term $\Pr(s_i|s_i^l,h_t)$, stored in the octree belief, is the \emph{weight} that sums up to 1 for all $s_i\in\Sspace_i$.

\textbf{Observation space $\hat{\Ospace}$ and function $\hat{O}$.}
For the purpose of planning, we again use the assumption that an object is contained within a single voxel (yet at resolution level $l$). Then, given state $\hat{s}'$, the abstract observation $o_i^l$ is regarded as a voxel-label pair $(s_i^l, d(s_i^l))$. Since it is computationally expensive to sum out all object states, we approximate the observation model by ignoring objects other than $i$:
\begin{align}
&\Pr(o_i^l|\abst{s}',a,h_t)=\Pr(d(s_i^l)|\hat{s}',a,h_t)\\
&\qquad\approx\Pr(d(s_i^l)|s^l_i,s_r,a,h_t)\\\
&\qquad=\sum_{s_i\in\phi_i^{-1}(s_i^l)}\Pr(d(s_i^l)|s_i,s_r,a)\Pr(s_i|s_i^l,h_t).
\end{align}
This resembles the abstract transition function, where $\Pr(d(s_i^l)|s_i,s_r,a)$
is the ground observation function, and $\Pr(s_i|s_i^l,h_t)$ is again the
weight.

For practical POMDP planning, it can be inefficient to sample from this abstract
observation model if $l$ is large. In our implementation, we approximate this distribution by Monte Carlo sampling \cite{shapiro2003monte}: We sample $k$ ground states from $\phi_i^{-1}(s_i^l)$ according to their weights.\footnote{We tested $k=10$ and $k=40$ and observed similar search performance. We used $k=10$ in our experiments.} Then we set $d(s_i^l)=i$ if the majority of these samples have $d(s_i)=i$, and $d(s_i^l)=\textsc{Free}$ otherwise. A similar approach is used for sampling from the abstract transition model.

\textbf{Reward function $R$.} The reward function is the same as the one in ground 3D-MOS, since computing the reward only depends on the robot state which is not abstracted and the abstract action space consists of the same primitive actions as 3D-MOS. Therefore, solving an abstract 3D-MOS is solving the same task as the original 3D-MOS.

\begin{algorithm}[t]
\caption{MR-POUCT $(\mathcal{P},b_t,d)\rightarrow \abst{a}$}
\label{alg:planning}
\SetKwInOut{Input}{input}
\SetKwInOut{Output}{output}
\SetKwFunction{plan}{Plan}
\SetKwFunction{solve}{PO-UCT}
\SetKwFunction{simulate}{Simulate}
\SetKwFunction{buildgenerator}{GenerativeFunction}
\SetKwFunction{rollout}{Rollout}
\SetKwProg{myproc}{procedure}{}{}
\Input{$\mathcal{P}$: a set of abstract 3D-MOS instances at different resolution levels; $b_t$: belief at time $t$; $d$: planning depth}
\Output{$\hat{a}$: an action in the action space of some $P_l\in\mathcal{P}$}
\myproc{\plan{$b_t$}}{
\ForEach{$P_l\in\mathcal{P}$ in parallel}{
  \tcp{Recall that $P_l=\langle \abst{\Sspace},\abst{\Aspace},\abst{\Ospace},\abst{T},\abst{O},R,\gamma,l \rangle$}
  $\mathcal{G}\gets$ GenerativeFunction($P_l$)\;
  $Q_P(b_t, \abst{a})\gets$ POUCT($\mathcal{G},h_t,d$)\;
}
$\abst{a}\gets\argmax_{\abst{a}}\{Q_P(b_t,\abst{a}) | P\in\mathcal{P}\}$\;
\KwRet{$\abst{a}$}
}
\end{algorithm}

\subsection{Multi-Resolution Planning Algorithm}
Abstract 3D-MOS is smaller than the original 3D-MOS which may provide benefit
in online planning. However, it may be difficult to define a single resolution
level, due to the uncertainty of the size or shape of objects, and the unknown
distance between the robot and these objects.

Therefore, we propose to solve a number of abstract 3D-MOS problems in parallel,
and select an action from $\abst{\Aspace}$ with the highest value for execution.
The algorithm is formally presented in Algorithm \ref{alg:planning}.  The set of
abstract 3D-MOS problems, $\mathcal{P}$, can be defined based on the
dimensionality of the search space and the particular object search
setting. Then, it is straightforward to define a \emph{generative function}
$\mathcal{G}(\hat{s},\hat{a})\rightarrow(\hat{s}',\hat{o},r)$ from an abstract
3D-MOS instance $P$ using its transition, observation and reward functions.
POUCT uses $\mathcal{G}$ to build a search tree and plan the next action. Thus,
all problems in $\mathcal{P}$ are solved online in parallel, each by a separate
POUCT. The final action with the highest value $Q_P(b_t,\hat{a})$ in its
respective POUCT search tree is chosen as the output (see
\cite{silver2010monte} for details on POUCT).
% \footnote{We refer the reader to
% \citet{silver2010monte} for details of the POUCT algorithm.}.
We call this
algorithm Multi-Resolution POUCT (MR-POUCT).

%%%%%%%%%%%%%%%%%%%
%%% Experiments %%%
%%%%%%%%%%%%%%%%%%%

\section{Experiments}

We assess the hypothesis that our approach, MR-POUCT, improves the robot's
ability to efficiently and successfully find objects especially in large search
spaces. We conduct a simulation evaluation and a study on a real robot.

\subsection{Simulation Evaluation}

\textbf{Setup.} We implement our approach in a simulated environment designed to reflect the essence of the 3D-MOS domain (Figure~\ref{fig:sim_env}). Each simulated problem instance is defined by a tuple $(m,n,d)$, where the search region $G$ has size $|G|=m^3$ with $n$ randomly generated, randomly placed objects. The on-board camera projects a viewing frustum with 45 degree FOV angle, an 1.0 aspect ratio, a minimum range of 1 grid cell, and a maximum range of $d$ grid cells. Hence, we can increase the difficulty of the problem by increasing $m$ and $n$, or by reducing the percentage of voxels covered by a viewing frustum through reducing the FOV range $d$. Occlusion is simulated using perspective projection and treating each grid cell as a point.%; The same occlusion mechanism is employed for all trials.

There are two primitive $\MOVE$ actions per axis (e.g. $+z$, $-z$) that each moves the robot along that axis by one grid cell. There are two $\LOOK$ actions per axis, one for each direction. Finally, a $\DETECT$ action is defined that declares all not-yet-found objects within the viewing frustum as found. Thus, the total number of primitive actions is $13$.
$\MOVE$ and $\LOOK$ actions have a step cost of -1.
% a $\MOVEOP$ accumulates the cost of the corresponding $\MOVE$ actions.
A successful $\DETECT$ receives +1000 while a failed attempt receives -1000. A $\DETECT$ action is successful if part of a new object lies within the viewing frustum.
%more explanation?
If multiple new objects are present within one viewing frustum when the $\DETECT$ is taken, only the maximum reward of $+1000$ is received. The task terminates either when the total planning time limit is reached or $n$ $\DETECT$ actions are taken.

\begin{figure}[t]
\centering
\includegraphics[width=1.0\linewidth]{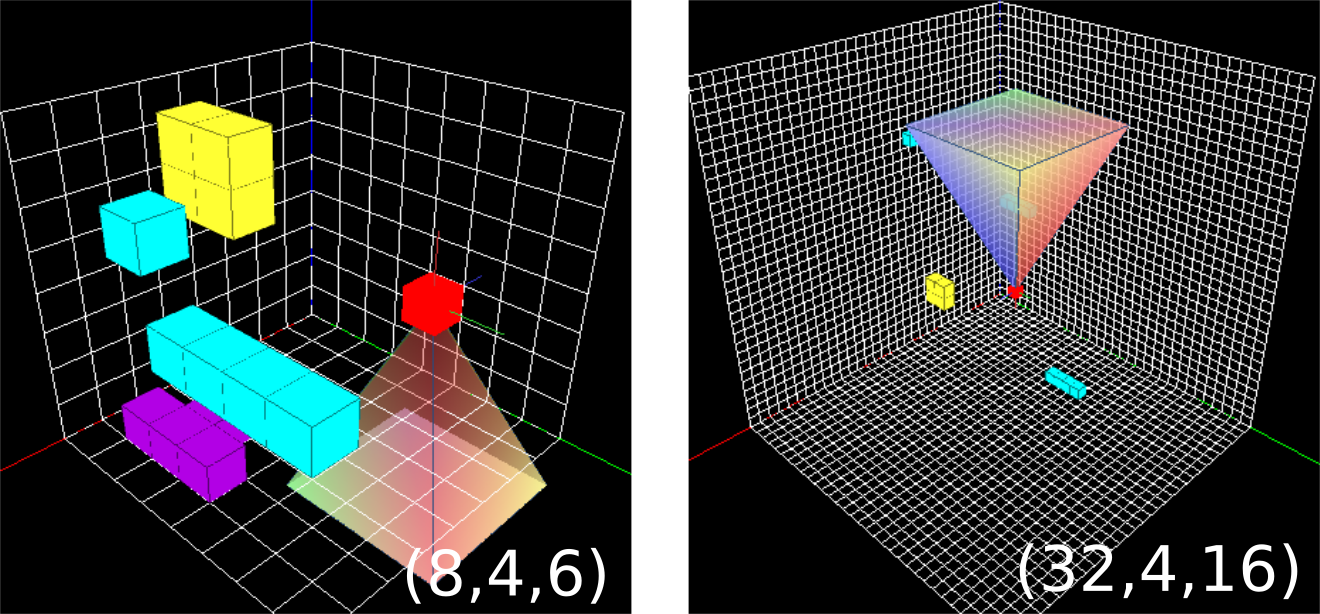}
\caption{Simulated environment for 3D object search. The robot (represented as a red cube) can project a viewing frustum to observe the search space, where objects are represented by sets of cubes. Search space size scales from $4^3$ to $32^3$. The tuple $(m,n,d)$ at lower-right defines the problem instance.}
\label{fig:sim_env}
\end{figure}

\begin{figure*}[t]
\centering
\includegraphics[width=0.95\linewidth]{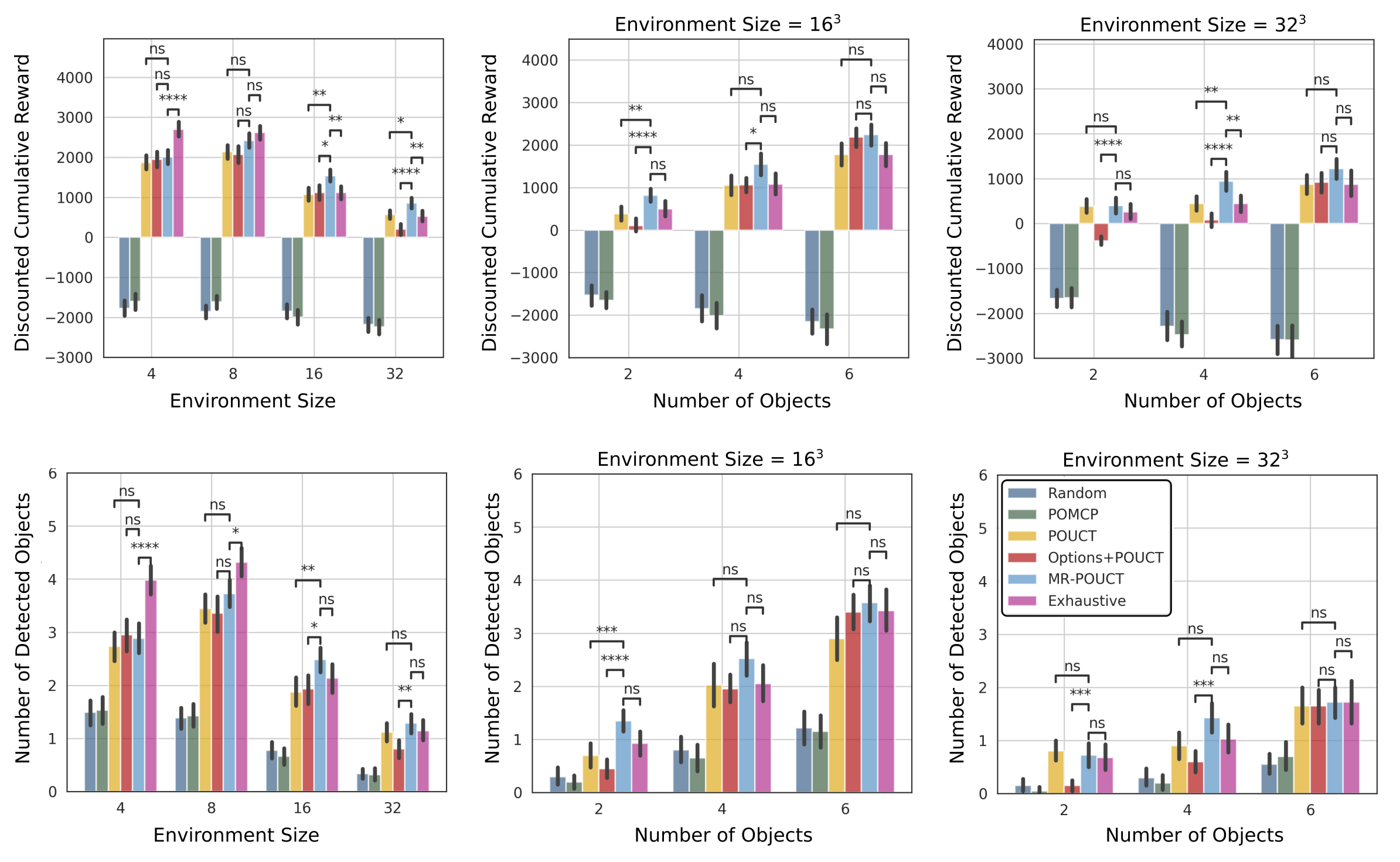}
\caption{Discounted cumulative reward and number of detected objects as the environment size ($m$) increases and as the of number of objects ($n$) increases.  Exhaustive search performs well in small-scale environments (4 and 8) where exploration strategy is not taken advantage of. In large environments, our method MR-POUCT performs better than the baselines in most cases. The error bars are 95\% confidence intervals. The level of statistical significance is shown, comparing MR-POUCT against POUCT, Options+POUCT, and Exhaustive, respectively, indicated by \texttt{ns} ($p>0.05$), * ($p\leq 0.05$), ** ($p\leq 0.01$), *** ($p\leq 0.001$), **** ($p\leq 0.0001$).}
  \label{fig:scalability_overall}
\end{figure*}

\begin{figure*}[t]
\centering
\includegraphics[width=\linewidth]{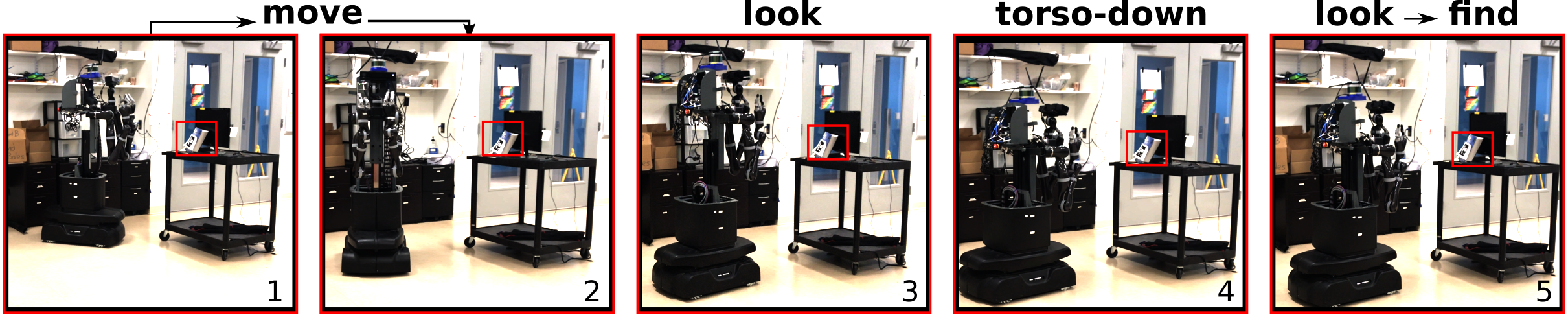}
\caption{ Example action sequence produced by the proposed approach. The mobile robot
  first navigates in front of a portable table (1-2). It then takes a $\LOOK$ action to observe the space in front (3), and no target is observed since the torso is too high. The robot then decides to lower its torso (4), takes another $\LOOK$ action in the same direction, and then $\DETECT$ to mark the object as found (5). This sequence of actions demonstrate that our algorithm can produce efficient search strategies in real world scenarios.}
\label{fig:seq}
% \vspace{-0.4cm}
\end{figure*}

\begin{figure}[thb]
\centering
\includegraphics[width=0.872\linewidth]{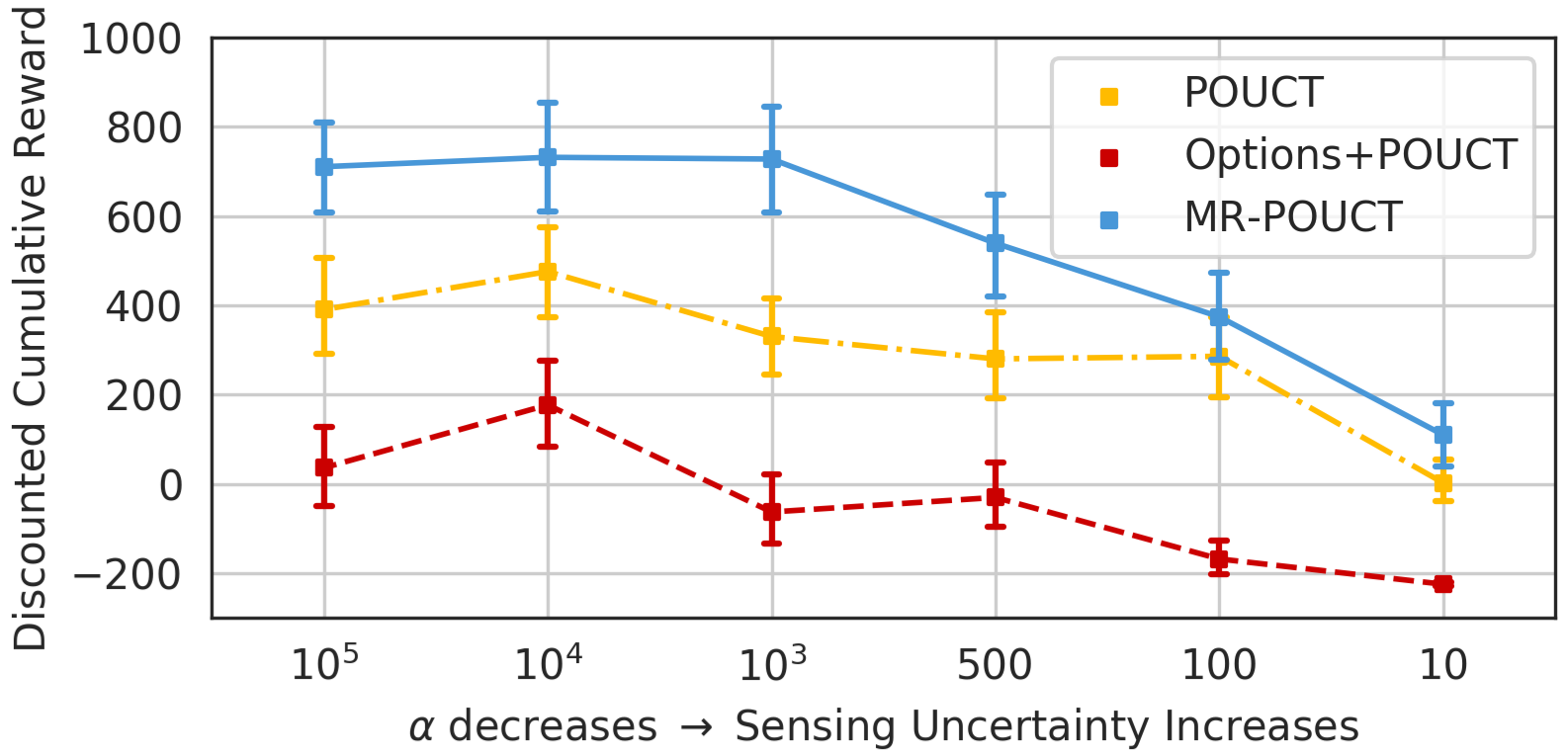}
\caption{Discounted cumulative reward with 95\% confidence interval as the sensing uncertainty increases, aggregating over the $\beta$ settings.}
  \label{fig:quality_results}
  \vspace{-1em}
\end{figure}

\textbf{Baselines.} We compare our approach (\emph{MR-POUCT}) with the following baselines:
% \begin{itemize}
% \item
  \emph{POUCT} uses the octree belief but solves the ground POMDP directly using the original POUCT algorithm. %\cite{silver2010monte}
  % without using resolution hierarchy in the octree belief.
\emph{Options+POUCT} uses the octree belief and a resolution hierarchy, but  only the motion action abstraction (i.e. $\MOVEOP$ options) is used, meaning that the agent can move for longer distances per planning step but do not make use of state and observation abstractions.
\emph{POMCP} uses a particle belief representation which is subject to particle deprivation. Uniform random rollout policy is used for all POMDP-based methods. \emph{Exhaustive} uses a hand-coded exhaustive policy, where the agent traverses every location in the search environment. At every location, the agent takes a sequence of $\LOOK$ actions, one in each direction. Finally, \emph{Random} executes actions at uniformly at random.
% \end{itemize}

Each algorithm begins with uniform prior and is allowed a maximum of 3.0s for planning each step. The total amount of allowed planning time plus time spent on belief update is 120s, 240s, 360s, and 480s for environment sizes ($m$) of 4, 8, 16, or 32, respectively. Belief update is not necessary for \emph{Exhaustive} and \emph{Random}. The maximum number of planning steps is 500. The discount factor $\gamma$ is set to $0.99$. For each $(m,n,d)$ setting, 40 trials (with random world generation) are conducted.

\textbf{Results.}
We evaluate the scalability of our approach with 4 different settings of search space size $m\in\{4,8,16,32\}$ and 3 settings of number of objects $n\in\{2,4,6\}$, resulting in 12 combinations. The FOV range $d$ is chosen such that the percentage of the grids covered by one projection of the viewing frustum decreases as the world size $m$ increases.\footnote{The maximum FOV coverage for $m=4,8,16,$ and $32$ is $17.2\% (d=4), 8.8\% (d=6), 4.7\% (d=10),$ and $2.6\% (d=16)$, respectively.}
The sensor is assumed to be near-perfect, with $\alpha=10^5$ and $\beta=0$.
We measure the discounted cumulative reward, which reflects both the search efficiency and effectiveness, as well as the number of objects found per trial.

Results are shown in Figure~\ref{fig:scalability_overall}. Particle deprivation
happens quickly due to large observation space, and the behavior degenerates to
a random agent, %as in \cite{silver2010monte,wandzel2019multi},
causing {POMCP} to perform poorly. In small-scale domains, the \emph{Exhaustive} approach works well, outperforming the POMDP-based methods. We find that in those environments, the FOV can capture a significant portion of the environment, making exhaustive search desirable.
The POMDP-based approaches are competitive or better in the two largest search environments ($m=16$ and $m=32$). In particular, MR-POUCT outperforms \emph{Exhaustive} in all test cases in the larger environments, with greater margin in discounted cumulative reward; \emph{Exhaustive} takes more search steps but is less efficient.
When the search space contains fewer objects, {MR-POUCT} and {POUCT} show more resilience than {Options+POUCT}, with {MR-POUCT} performing consistently better.
This demonstrates the benefit of planning with the resolution hierarchy in octree belief especially in large search environments.

We then investigate the performance of our method with respect to changes in sensing uncertainty, controlled by the parameters $\alpha$ and $\beta$ of the observation model. According to the belief update algorithm in Section~\ref{sec:octree:belief_update}, a noisy but functional sensor should increase the belief $\Val_t^i(g)$ for object $i$ if an observed voxel at $g$ is labeled $i$, while decrease the belief if labeled \textsc{Free}. This implies that a properly working sensor should satisfy $\alpha > 1$ and $\beta < 1$. We investigate on 5 settings of $\alpha\in\{10,100,500,10^3,10^4,10^5\}$ and 2 settings of $\beta\in\{0.3, 0.8\}$. A fixed problem difficulty of $(16,2,10)$ is used to conduct this experiment.
Results in Figure~\ref{fig:quality_results} show that {MR-POUCT} is consistently better in all parameter settings. We observe that $\beta$ has almost no impact to any algorithm's performance as long as $\beta < 1$, whereas decreasing $\alpha$ changes the agent behavior such that it must decide to $\LOOK$ multiple times before being certain.

\subsection{Demonstration on a Torso-Actuated Mobile Robot}

We demonstrate that our approach is scalable to real world settings by implementing the 3D-MOS problem as well as MR-POUCT for a mobile robot setting. We use the Kinova MOVO Mobile Manipulator robot, which has an
actuated torso with an extension range between around 0.05m and 0.5m, which facilitates a 3D action space. The robot operates in a lab environment, which is decomposed into two \emph{search regions} $G_1$ and $G_2$ of size roughly 10m$^2\times$ 2m (Figure.~\ref{fig:seq}), each with a semantic label (``shelf-area'' for $G_1$ and ``whiteboard-area'' for $G_2$). The robot is tasked to look for $n_{G_1}$ and $n_{G_2}$ objects in each search region sequentially, where objects are represented by paper AR tags that could be in clutter or not detectable at an angle. The robot instantiates an instance of the 3D-MOS problem once it navigates to a search region. In this 3D-MOS implementation, the $\MOVE$ actions are implemented based on a topological graph on top of a metric occupancy grid map. The neighbors of a graph node form the motion action space when the robot is at that node.
% Since this motion action space is already an abstraction over the metric grid map, we do not impose $\MOVEOP$ to the Abstract 3D-MOS in this case.
The robot can take $\LOOK$ action in 4 cardinal directions in place and receive volumetric observations; A volumetric observation is a result of downsampling and thresholding points in the corresponding point cloud. The robot was able to find 3 out of 6 total objects in the two search regions in around 15 minutes. One sequence of actions (Figure~\ref{fig:seq}) shows that the robot decides to lower its torso in order to $\LOOK$ and $\DETECT$ an object.\footnote{Video footage with visualization of volumetric observations and octree belief update is available at \href{https://zkytony.github.io/3D-MOS/}{https://zkytony.github.io/3D-MOS/}.} A failure mode is that the object may not be covered by any viewpoint and thus not detected; this can be improved with a denser topological map, or by considering destinations of $\MOVE$ actions sampled from the continuous search region.

% \bstnote{make sure we change to a youtube link in the final camera ready.}

%%%%%%%%%%%%% END TABLES

%%%%%%%%%%%%%%%%% CONCLUSION %%%%%%%%%%%%
\section{Conclusion}
\label{sec:conclusion}
We present a POMDP formulation of multi-object search in 3D with volumetric observation space and solve it with a novel multi-resolution planning algorithm. Our evaluation demonstrates that such challenging POMDPs can be solved online efficiently and scalably with practicality for a real robot by extending existing general POMDP solvers with domain-specific structure and belief representation.

% Our work represents an effort in extending previous
% POMDP-based methods to object search to 3D.

One limitation of the presented work is that the assumption of object independence, though beneficial computationally, may discard useful object dependence information in some cases. Optimal search for correlated objects becomes important. In addition, we do not explicitly reason over object geometry in the observation model. Considering belief over geometric appearances is a challenging future direction. Finally, incorporating a heuristic rollout policy may be a promising direction for more realistic object search problems while sacrificing optimality.

\section*{Acknowledgements}
{\small The authors thank Selena Ling for help with the simulator.  This work was supported by the National Science Foundation under grant number IIS1652561, ONR under grant number N00014-17-1-2699, the US Army under grant number W911NF1920145, Echo Labs, STRAC Institute, and Hyundai.}

%%%%%%%%%%%%%%
% (computed using the plotting_scalability.py)
% For world of 32x32x32,
% The mean discounted cum.reward of MR-POUCT is: 854.95265 ;
% The mean discounted cum.reward of Options-POUCT is: 203.38074 ;
% The mean discounted cum.reward of POUCT is: 565.9957;
% So MR-POUCT's improvement over Options-POUCT is 320%
% and over POUCT is 51%.
%
% For world of 16x16x16
% The mean discounted cum.reward of MR-POUCT is 1536.554767669474
% The mean discounted cum.reward of Options+POUCT is 1118.952102888304
% The mean discounted cum.reward of POUCT is 1072.6382089647166
% MR-POUCT improvement over Options+POUCT is 37.3%,
% over POUCT is 43%

\bibliographystyle{IEEEtranN}
\bibliography{IEEEabrv,references}

\end{document}